\documentclass[a4paper, 12pt, one column]{article}
\def\rset{\mathbb{R}}

\def\param{w}

\def\Bm{\mathrm{B}}
\usepackage{bbm}
\newcommand{\True}{\mathrm{True}}
\newcommand{\False}{\mathrm{False}}

\newcommand{\indiacc}[1]{\1_{\{#1\}}}

\newcommand{\1}{\ensuremath{\mathbbm{1}}}

\newcommand{\wtn}{\widetilde{\nabla}}
\def\objfunc{f}

\def\sign{\operatorname{sign}}

\usepackage{xargs}
\newcommandx{\CPE}[3][1=]{\PE_{#1}\left[\left. #2 \, \right| #3 \right]} 

\def\PE{\mathbb{E}}

\usepackage[utf8x]{inputenc}
\usepackage[T1]{fontenc}
\usepackage[ruled,linesnumbered]{algorithm2e}
\usepackage[top=1.3cm, bottom=2.0cm, outer=2.5cm, inner=2.5cm, heightrounded,
marginparwidth=1.5cm, marginparsep=0.4cm, margin=2.5cm]{geometry}

\usepackage{graphicx} 
\usepackage[colorlinks=False]{hyperref} 
\usepackage{amsmath}  
\usepackage{amsfonts} %
\usepackage{amssymb}  %
\usepackage{amsthm}
\usepackage{cleveref}
\usepackage[authoryear]{natbib}
\setcitestyle{authoryear,open={(},close={)}}

\usepackage{tikz}
\usetikzlibrary{positioning}
\usetikzlibrary{shapes.misc}
\usetikzlibrary{automata,arrows,positioning,calc}

\theoremstyle{plain}
\newtheorem{theorem}{Theorem}[section]

\newtheorem{lemma}[theorem]{Lemma}

\theoremstyle{definition}

\newtheorem{assumption}{A.}
\theoremstyle{remark}
\newtheorem{remark}[theorem]{Remark}

\title{Boolean Logic as an Error feedback mechanism}
\date{}
\author{Louis Leconte\\
LISITE, Isep, Sorbonne University\\
Math. and Algorithmic Sciences Lab, Huawei Technologies, Paris, France}
\begin{document}
\maketitle
\begin{abstract}
    The notion of Boolean logic backpropagation was introduced to build neural networks with weights and activations being Boolean numbers. Most of computations can be done with Boolean logic instead of real arithmetic, both during training and inference phases. But the underlying discrete optimization problem is NP-hard, and the Boolean logic has no guarantee. In this work we propose the first convergence analysis, under standard non-convex assumptions.
\end{abstract}

\section{Introduction}
Training machine learning models can often be a very challenging process, requiring significant computational resources and time. The use of DNNs on computing hardware such as mobile and IoT devices is becoming increasingly important. IoT devices often have limitations in terms of memory and computational capacity. Quantization is a potential solution to this problem \citep{Courbariaux2015, chmiel2021logarithmic, leconte2023askewsgd}. And in particular, Binary Neural Networks (BNNs) is a remarkably promising direction because it reduces both memory and inference latency simultaneously \citep{nguyen2023boolean}.

Formaly, BNN training can be formulated as minimising the training loss with binary weights, i.e.,
\begin{equation}
\label{eq:optimization-problem}
\min_{\param \in \mathbf{Q}} \objfunc(\param); \, \, \objfunc(\param)= \PE_{(\param, y) \sim p_{\operatorname{data }}}[\ell(\mathrm{NN}(x, \param), y)],
\end{equation}
where $\mathbf{Q} = \{\pm1\}^d$ is a discrete codebook, $d$ is the number of parameters (network weights and biases), $n$ is the total number of clients, $\ell$ is the training loss (e.g., cross-entropy or square loss), $\mathrm{NN}(x,\param)$ is the DNN prediction function, $p_{\operatorname{data}}$ is the training distribution.
The quantization constraints in the above program make it an extremely difficult task: the underlying optimization problem is non-convex, non-differentiable, and combinatorial in nature.

To the best of our knowledge, in the quantized neural network literature and in particular BNN, one can only prove the convergence up to an irreducible error floor \cite{li2017training}. This idea has been extended to SVRG \cite{de2018high}, and recently to SGLD in \cite{zhang2022low}, which is also up to an error limit.

In this work we provide complexity bounds for Boolean Logic \citep{nguyen2023boolean} in a smooth non-convex environment. We first recap the simplified mechanism of a given Boolean Logic (noted as $\Bm$) in \Cref{algo:BooleanTraining}. In the forward pass, at iteration $t$, input of layer $l$, $x^{l,t}$, is buffered for later use in the backward, and the $j$th neuron output at $k$th sample is computed as:
\begin{equation}\label{eq:Forward}
	x_{k,j}^{l+1,t} = w_{0, j}^l + \sum_{i=1}^m \Bm (x_{k, i}^l, w_{i, j}^l),
\end{equation} 
$\forall k \in [1,K], \forall j \in [1,n]$ where $K$, $m$, $n$ are, respectively, the training mini-batch, layer input and output size.
\begin{algorithm}[!t]
	\caption{Pseudo-code for Boolean training with $\Bm=XNOR$.}
	\label{algo:BooleanTraining}
	\SetKwInOut{Input}{Input}
	\SetKwInOut{Output}{Output}
	\SetKwBlock{Loop}{Loop}{end}
	\SetKwBlock{Initialize}{Initialize}{end}
	\SetKwFor{When}{When}{do}{end}
	\SetKwFunction{Wait}{Wait}
	\SetAlgoLined
	\SetNoFillComment
	\Input{Learning rate $\eta$, nb iterations $T$\;}
	\Initialize{	
	$m_{i,j}^{l,0} = 0$; $\beta^0 = 1$\;
	}
	\For{$t = 0,\dots, T-1$}{
	\tcc{\textbf{1. Forward}}
	Receive and buffer $x^{l,t}$\;
	Compute $x^{l+1,t}$ following \Cref{eq:Forward}\; \label{algoTrain:line:forward}
	\tcc{\textbf{2. Backward}}
	Receive $g^{l+1,t}$\;
	\tcc{\textbf{2.1  Backpropagation}}
	Compute and backpropagate $g^{l,t}$ following \Cref{eq:Bprop}\; \label{algoTrain:line:backward}
	\tcc{\textbf{2.2  Weight update}}
	$C_{\textrm{tot}} := 0$, $C_{\textrm{kept}} := 0$\;
	\ForEach{$w_{i,j}^{l}$}{
		Compute $q_{i,j}^{l,t+1}$ following \Cref{eq:Optim}\;  \label{algoTrain:line:optim}
		Update $m_{i,j}^{l,t+1} = \beta^t m_{i,j}^{l,t} + \eta^t q_{i,j}^{l,t+1}$\;
		$C_{\textrm{tot}} \gets C_{\textrm{tot}} + 1$\;
		\eIf{$XNOR(m_{i,j}^{l,t+1}, w_{i,j}^{l,t}) = \True$}{ \label{algoTrain:line:if}
			$w_{i,j}^{l,t+1} \gets \neg w_{i,j}^{l,t}$ \tcc*[r]{FLIP}
			$m_{i,j}^{l,t+1} \gets 0$\;
		}{
			$w_{i,j}^{l,t+1} \gets w_{i,j}^{l,t}$  \tcc*[r]{NO FLIP}
			$C_{\textrm{kept}} \gets C_{\textrm{kept}} + 1$\;
		}\label{algoTrain:line:endif}
	}
	Release buffer $x^{l,t}$\;
	Update $\beta^{t+1} \gets C_{\textrm{kept}}/C_{\textrm{tot}}$ \;
	Update $\eta^{t+1}$\;
	}
\end{algorithm}

In the backward pass, layer $l$ receives $g^{l+1,t}$ from downstream layer $l+1$. Then, backpropagated signal $g^{l,t}$ (line \ref{algoTrain:line:backward} in \Cref{algo:BooleanTraining}), is computed following \cite{nguyen2023boolean} as:
\begin{equation}\label{eq:Bprop}
	g_{k,i}^{l,t} = \sum_{j=1}^n \indiacc{g_{k,i,j}^{l,t} = True}|g_{k,i,j}^{l,t}| - \sum_{j=1}^n \indiacc{g_{k,i,j}^{l,t} = False}|g_{k,i,j}^{l,t}|,
\end{equation} 
$\forall k \in [1,K], \forall i \in [1,m]$, where $g_{k,i,j}^{l,t}$ is given according to \cite{nguyen2023boolean} for the utilized logic $\Bm$. Optimization signal at line \ref{algoTrain:line:optim} in \Cref{algo:BooleanTraining} is given according to \cite{nguyen2023boolean} as:
\begin{equation}\label{eq:Optim}
	q_{i,j}^{l,t+1} = \sum_{k=1}^K \indiacc{q_{i,j,k}^{l,t} = \True}|q_{i,j,k}^{l,t}| \\ - \sum_{k=1}^K \indiacc{q_{i,j,k}^{l,t} = \False}|q_{i,j,k}^{l,t}|,
\end{equation}
$\forall i \in [1,m], \forall j \in [1,n]$. Finally, the weights are updated in lines \ref{algoTrain:line:if}--\ref{algoTrain:line:endif} of \Cref{algo:BooleanTraining} following the rule formulated in \cite{nguyen2023boolean}.

We now introduce an abstraction to model the optimization process and prove convergence of the mechanism detailed in \Cref{algo:BooleanTraining}.

\section{Continuous Abstraction of \cite{nguyen2023boolean}}\label{sec:analysis-prel}

Boolean optimizer is discrete, proving its convergence directly is a hard problem. The idea is to find a continuous equivalence so that some proof techniques existing from the BNN and quantized neural networks literature can be employed.  

In existing frameworks, quantity $\wtn f(\cdot)$ denotes the stochastic gradient computed on a random mini-batch of data. Boolean Logic does not have the notion of gradient, it however has an optimization signal ($q_{i,j}^{l,t}$ in \Cref{algo:BooleanTraining}) that plays the same role as $\wtn f(\cdot)$. Therefore, these two notions, i.e., continuous gradient and Boolean optimization signal, can be encompassed into a generalized notion. That is the root to the following continuous relaxation in which $\wtn f(\cdot)$ stands for the optimization signal computed on a random mini-batch of data.

For reference, the original Boolean optimizer as formulated in in the previous section is summarized in \Cref{alg:BoolOptim} in which $\texttt{flip}(\param_t, m_{t+1})$ flips weight and $\texttt{reset}(\param_t, m_{t+1})$ resets its accumulator when the flipping condition is triggered.
\begin{algorithm}[h]
	\caption{Boolean optimizer}\label{alg:BoolOptim}
	$m_{t+1} \gets \beta_{t} m_{t} + \eta q_t$ \;
	$\param_{t+1} \gets \texttt{flip}(\param_t, m_{t+1})$\;
	$m_{t+1} \gets \texttt{reset}(\param_t, m_{t+1})$\;
\end{algorithm}

\begin{algorithm}[h]
	\caption{Equivalent formulation of Boolean optimizer}\label{alg:EquivOptim}
	\KwData{$Q_0, Q_1$ quantizer}
	$m_t \gets \eta \wtn f(\param_t) + e_t $\;
	$\Delta_t \gets Q_1(m_t, \param_t)$\;
	$\param_{t+1} \gets Q_0(\param_t - \Delta_t)$\;
	$e_{t+1} \gets m_t - \Delta_t$\;
\end{algorithm}

\Cref{alg:EquivOptim} describes an equivalent formulation of Boolean optimizer. Therein, $Q_0$, $Q_1$ are quantizers which are specified in the following. Note that EF-SIGNSGD (SIGNSGD with Error-Feedback) algorithm from \cite{errorfeedback} is a particular case of this formulation with $Q_0()=\operatorname{Identity}()$ and $Q_1()=\sign()$.
For Boolean Logic abstraction, they are given by:

\begin{equation}
\label{eq:quantizers}
	\begin{cases}
		Q_1(m_t, \param_t) = \param_t (\textrm{ReLu}(\param_t m_t-1)+\frac{1}{2} \sign(\param_t m_t-1) +\frac{1}{2}), \\
		Q_0(\param_t) = \sign(\param_t).
	\end{cases}
\end{equation}

The combination of $Q_1$ and $Q_0$ is crucial to take into account the reset property of the accumulator $m_t$. Indeed in practice, $\Delta_t:=Q_1({m_t}, \param_t)$ is always equal to $0$ except when $| m_t | > 1$ and $\sign(m_t)=\sign(\param_t)$ (i.e., when the flipping rule is applied). As $w_t$ has only values in $\{\pm 1\}$, $Q_0$ acts as identity function, except when $\Delta_t$ is non-zero (i.e., when the flipping rule is applied). With the choices \eqref{eq:quantizers}, we can identify $\texttt{flip}(\param_t, m_{t}) = Q_0(\param_t - Q_1(m_t, \param_t))$. We do not have closed-form formula for $\texttt{reset}(\param_t, m_{t+1})$ from \Cref{alg:BoolOptim}, but the residual errors $e_t$ play this role. Indeed, $e_{t+1} = m_t$ except when $\Delta_t$ is non-zero (i.e., when the flipping rule is applied and $e_{t+1}$ is equal to $0$).

The main difficulty in the analysis comes from the parameters quantization $Q_0()$. Indeed, we can follow the derivations in Appendix B.3 from \cite{errorfeedback} to bound the error term $\PE{\Vert e_t \Vert^2}$, but we also have additional terms coming from the quantity:
\begin{equation}\label{eq:ht}
	h_t = Q_0(\param_t - Q_1(m_t, \param_t)) - (\param_t - Q_1(m_t, \param_t)).
\end{equation}

\section{Non-convex analysis}
In the following, we prove that Boolean logic optimizer \citep{nguyen2023boolean} converges towards a first-order stationary point, as $T$ the number of global epochs grows.

\subsection{Preliminaries}
Our analysis is based on the following standard non-convex assumptions on $f$:
\begin{assumption} 
	\label{assum:uniflowerbound}
	Uniform Lower Bound: There exists $f_* \in \mathbb{R}$ s.t. $f(\param) \geq f_*$, $\forall \param \in \mathbb{R}^d$.
\end{assumption}
\begin{assumption} 
	\label{assum:smoothgradients}
	Smooth Derivatives: The gradient $\nabla f(\param)$ is $L$-Lipschitz continuous for some $L>0$, i.e., $\forall \param, \forall v \in \mathbb{R}^d$:
	$
		\left\|\nabla f(\param)-\nabla f(v)\right\| \leq L\|\param-v\| .
	$
\end{assumption}
\begin{assumption} 
	\label{assum:boundedvariance}
	Bounded Variance: The variance of the stochastic gradients is bounded by some $\sigma^2>0$, i.e., $\forall \param \in \mathbb{R}^d$:
	$
		\PE[\wtn f(\param)] = \nabla f(\param)
	$ and $
		\PE[\|\wtn f(\param)\|^2] \leq \sigma^2 .
	$
\end{assumption}
\begin{assumption}
	\label{assum:alwaysflip}
	Compressor: There exists $\delta<1$ s.t. $\forall\param, \forall v \in \mathbb{R}^d$, $\Vert Q_1(v,\param)-v \Vert^2 \leq \delta \Vert v \Vert^2$.
\end{assumption}
\begin{assumption}
	\label{assum:boundedaccumulator}
	Bounded Accumulator: There exists $\kappa \in \mathbb{R}^*_+$ s.t. $\forall t$ and $\forall i \in [d]$, we have $|m_t|_i \le \eta \kappa$.
\end{assumption}
\begin{assumption}
	\label{assum:stoflip}
	Stochastic Flipping Rule: For all $\param \in \rset$, we have $\CPE{Q_0(\param)}{\param}=\param$.
\end{assumption}
In particular, \Cref{assum:boundedaccumulator} and \Cref{assum:stoflip} enable us to obtain $\PE[h_t] = 0$ and to bound the variance of $h_t$.
Based on all these assumptions, we prove the following:
\begin{theorem}
	\label{theorem:cvgquant}
	Assume \Cref{assum:uniflowerbound} to \Cref{assum:stoflip}. Boolean Logic applied to Boolean weights $w$ converges at rate:
	\begin{equation}
		\frac{1}{T}\sum_{t=0}^{T-1} \PE{\left\|\nabla f\left(w_{t}\right)\right\|^{2}} \le  \frac{A^*}{T\eta} + B^* \eta + C^* \eta^2 + Lr_d,
	\end{equation}
	where $A^* = {2(f(w_0)-f_*)}$, $B^*= 2L\sigma^2$, $C^*= 4L^2\sigma^2 \frac{\delta}{(1-\delta)^2}$, $r_d = \frac{d\kappa}{2}$.
\end{theorem}

\begin{remark}
	Our analysis is independent of the quantization function $Q_0$. We impose a weak assumption on $Q_0$ (assumption \Cref{assum:stoflip}), which holds for standard quantization methods such as stochastic rounding.
\end{remark}
\begin{remark}
	An important remark is that we only consider parameter quantization in the analysis. Nonetheless, our results remain valid when an unbiased quantization function is used to quantize computed gradients. Indeed, the stochastic gradients remain unbiased under such quantization methods. The only effect of the quantization would be an increased variance in the stochastic gradients.
\end{remark}

\begin{remark}
	Assumptions \ref{assum:uniflowerbound} to \ref{assum:boundedvariance} are standard. Assumptions \ref{assum:alwaysflip} to \ref{assum:stoflip} are non-classic but dedicated to Boolean Logic strategy. \Cref{assum:alwaysflip} is equivalent to assuming Boolean Logic optimization presents at least one flip at every iteration $t$. \Cref{assum:alwaysflip} is classic in the literature of compressed SGD \cite{errorfeedback, alistarh2017qsgd}. Moreover, \Cref{assum:boundedaccumulator} and \Cref{assum:stoflip} are not restrictive, but algorithmic choices. For example, rounding ($Q_0$ function) can be stochastic based on the value of the accumulator $m_t$. Similar to STE clipping strategy, the accumulator can be clipped to some pre-defined value $\kappa$ before applying the flipping rule to verify \Cref{assum:boundedaccumulator}.
\end{remark}

\begin{remark}
	Our proof assumes that the step size $\eta$ is constant over iterations. But in practice, we gently decrease the value of $\eta$ at some time steps. Our proof can  be adapted to this setting by defining a gradient accumulator $a_t$ such that $a_{t+1}=a_t+  q_t$. When $\eta$ is constant we recover the accumulation definition and we obtain $m_t= \eta a_t$. In the proposed algorithm, gradients are computed on binary weight $\param_t$ and accumulated in $a_t$. Then, one applies the flipping rule on the quantity $\Tilde{\param}_t = \eta a_t$ ($\Tilde{\param}_t=m_t$ when $\eta$ is constant), and one (may) reset the accumulator $a_t$.
\end{remark}

We start by stating a key lemma which shows that the residual errors $e_t$ maintained in \Cref{alg:EquivOptim} do not accumulate too much.

\begin{lemma}
    \label{lem:boundederror}
    Under \Cref{assum:boundedvariance} and \Cref{assum:alwaysflip}, the error can be bounded as $\PE[\Vert e_t \Vert^2] \leq \frac{2\delta(1+\delta)}{(1-\delta)^2}\eta^2\sigma^2$.
\end{lemma}

\begin{proof}
	We start by using the definition of the error sequence:
	\begin{align*}
		\Vert e_{t+1} \Vert^2 = \Vert Q_1(m_t, \param_t)- m_t \Vert^2.
	\end{align*}
	Next we make use of \Cref{assum:alwaysflip}:
	\begin{align*}
		\Vert e_{t+1} \Vert^2 \leq \delta \Vert m_t \Vert^2.
	\end{align*}
	We develop the accumulator update:
	\begin{align*}
		\Vert e_{t+1} \Vert^2 \leq \delta \Vert e_t + \eta \wtn f(\param_t) \Vert^2.
	\end{align*}
	We thus have a recurrence relation on the bound of $e_t$. Using Young’s inequality, we have that for any $\beta>0$,
	\begin{align*}
		\Vert e_{t+1} \Vert^2 \leq \delta (1+\beta) \Vert e_t \Vert^2 + \delta (1+\frac{1}{\beta}) \eta^2 \Vert \wtn f(\param_t) \Vert^2.
	\end{align*}
	Rolling the recursion over and using \Cref{assum:boundedvariance} we obtain:
	\begin{align*}
		\PE[\Vert e_{t+1} \Vert^2] \leq & \delta (1+\beta) \PE[\Vert e_t \Vert^2] + \delta (1+\frac{1}{\beta}) \eta^2 \PE[\Vert \wtn f(\param_t) \Vert^2] \\
		\leq                            & \delta (1+\beta) \PE[\Vert e_t \Vert^2] + \delta (1+\frac{1}{\beta}) \eta^2 \sigma^2                                      \\
		\leq                            & \sum_r^t (\delta (1+\beta))^r \delta (1+\frac{1}{\beta}) \eta^2 \sigma^2                                                  \\
		\leq                            & \frac{\delta(1+\frac{1}{\beta})}{1-\delta(1+\beta)} \eta^2 \sigma^2.
	\end{align*}
	Taking $\beta = \frac{1-\delta}{2\delta}$ and plugging it in the above bounds gives:
	\begin{align*}
		\PE[\Vert e_{t+1} \Vert^2] \leq \frac{2\delta(1+\delta)}{(1-\delta)^2}\eta^2\sigma^2.
	\end{align*}
\end{proof}

Then, the next Lemma allows us to bound the averaged norm-squared of the distance between the Boolean weight and 
$\param_t - Q_1(m_t, \param_t)$. We make use of the previously defined quantity \Cref{eq:ht} and have:
\begin{lemma}
    \label{lem:boundedht}
    Under assumptions \Cref{assum:boundedaccumulator} and \Cref{assum:stoflip}: $\PE[\Vert h_t \Vert^2] \leq \eta d \kappa$.
\end{lemma}

\begin{proof}
	Let consider a coordinate $i \in [d]$. $Q_0|_i$ as $-1$ or $+1$ for value with some probability $p_{i,t}$. For the ease of presentation, we will drop the subscript $i$. Denote $u_t := \param_t - Q_1(m_t, \param_t)$. Hence, $h_t$ can take value $(1-u_t)$ with some probability $p_t$ and $(-1-u_t)$ with probability $1-p_t$. Assumption \Cref{assum:stoflip} yields $2p_t-1=u_t$. Therefore, we can compute the variance of $h_t$ as follows:
	\begin{align*}
		\PE[\Vert h_t \Vert^2] & = \PE[\sum_i^d 1 + (\param_t - Q_1(m_t, \param_t))^2 -2Q_0(\param_t - Q_1(m_t, \param_t)(\param_t - Q_1(m_t, \param_t)] \\
		& = \sum_i^d ((1-u_t)^2p_t + (-1-u_t)^2(1-p_t)) \\
        & = \sum_i^d (1 + u_t^2 - 2u_t(2p_t-1)) \\
		& = \sum_i^d (1-u_t^2).
	\end{align*}
	The definition of $u_t$ leads to
	\begin{align*}
		1-u_t^2 & = 1-(1+Q_1(m_t, \param_t)^2-2\param_t Q_1(m_t, \param_t)) \\
		& =Q_1(m_t, \param_t)(2\param_t-Q_1(m_t, \param_t)).
	\end{align*}
	When $|m_t| \leq 1$ or $\sign(m_t) \ne \sign(\param_t)$, we directly have $Q_1(m_t, \param_t)(2\param_t-Q_1(m_t, \param_t)) = 0 \leq \eta \kappa$. When $|m_t| > 1$ and $\sign(m_t) = \sign(\param_t)$, we apply the definition of $Q_1$ to obtain:
	\begin{align*}
		Q_1(m_t, \param_t)(2\param_t-Q_1(m_t, \param_t)) & \leq m_t(2w_t - m_t) \\
		& \leq |m_t| \\
        & \leq \eta \kappa.
	\end{align*}
	Therefore, we can apply this result to every coordinate, and conclude that:
	\begin{equation*}
		\PE[\Vert h_t \Vert^2] \leq \eta d \kappa.
	\end{equation*}
\end{proof}

\subsection{Proof of \Cref{theorem:cvgquant}}
We now can proceed to the proof of \Cref{theorem:cvgquant}.
\begin{proof}
	Consider the virtual sequence $x_t=\param_t-e_t$. We have:
	\begin{align*}
		x_{t+1} & = Q_0(\param_t-\Delta_t) - (m_t-\Delta_t) \\
		& = (Q_0(\param_t-\Delta_t) + \Delta_t - e_t) - \eta \wtn f(\param_t).
	\end{align*}
	Considering the expectation with respect to the random variable $Q_0$ and the gradient noise, we have:
	\begin{align*}
		\CPE{x_{t+1}}{\param_t} = x_t - \eta \nabla f (\param_t).
	\end{align*}
	We consider $\PE_t[\cdot]$ the expectation with respect to every random process know up to time $t$. We apply the $L$-smoothness assumption \Cref{assum:smoothgradients}, and assumptions \Cref{assum:boundedvariance}, \Cref{assum:stoflip} to obtain:
	\begin{align*}
		\PE_t[f(x_{t+1})-f(x_t)] & \leq -\eta \langle \nabla f (x_t), \nabla f (\param_t) \rangle + \frac{L}{2} \PE_t[\Vert (Q_0(\param_t-\Delta_t) + \Delta_t) - \eta \wtn f(\param_t) - \param_t \Vert^2].
	\end{align*}
	We now reuse $h_t$ from \Cref{eq:ht} and simplify the above:
	\begin{align*}
		\PE_t[f(x_{t+1})-f(x_t)] & \leq -\eta \langle \nabla f (x_t), \nabla f (\param_t) \rangle + \frac{L}{2} \PE_t[\Vert h_t - \eta \wtn f(\param_t) \Vert^2] \\
		& \leq -\eta \langle \nabla f (x_t)-\nabla f (\param_t) + \nabla f (\param_t), \nabla f (\param_t) \rangle + \frac{L}{2} \PE_t[\Vert h_t - \eta \wtn f(\param_t) \Vert^2].
	\end{align*}
	Using Young’s inequality, we have that for any $\beta > 0$,
	\begin{align*}
		\PE_t[f(x_{t+1})-f(x_t)] \leq & -\eta \langle \nabla f (x_t)-\nabla f (\param_t) + \nabla f (\param_t), \nabla f (\param_t) \rangle \\
		& + \frac{L}{2}(1+\beta) \PE_t[\Vert h_t\Vert^2] + \frac{L}{2}\eta^2(1+\frac{1}{\beta})\sigma^2.
	\end{align*}
	Making use again of smoothness and Young's inequality we have:
	\begin{align*}
		\PE_t[f(x_{t+1})-f(x_t)] \leq & -\eta \Vert \nabla f (\param_t)\Vert^2 -\eta \langle \nabla f (x_t)-\nabla f (\param_t), \nabla f (\param_t) \rangle \\
		& + \frac{L}{2}(1+\beta) \PE_t[\Vert h_t\Vert^2] + \frac{L}{2}\eta^2(1+\frac{1}{\beta})\sigma^2 \\
		\leq & -\eta \Vert \nabla f (\param_t)\Vert^2 + \frac{\eta \rho}{2} \Vert \nabla f (\param_t) \Vert^2 + \frac{\eta}{2\rho} \Vert \nabla f(x_t) - \nabla f(\param_t) \Vert^2 \\
		& + \frac{L}{2}(1+\beta) \PE_t[\Vert h_t\Vert^2] + \frac{L}{2}\eta^2(1+\frac{1}{\beta})\sigma^2 \\
		\leq & -\eta \Vert \nabla f (\param_t)\Vert^2 + \frac{\eta \rho}{2} \Vert \nabla f (\param_t) \Vert^2 + \frac{\eta L^2}{2\rho} \underbrace{\Vert x_t - \param_t \Vert^2}_{\Vert e_t \Vert^2} \\
		& + \frac{L}{2}(1+\beta) \PE_t[\Vert h_t\Vert^2] + \frac{L}{2}\eta^2(1+\frac{1}{\beta})\sigma^2.
	\end{align*}
	Under the law of total expectation, we make use of \Cref{lem:boundederror} and \Cref{lem:boundedht} to obtain:
	\begin{align*}
		\PE[f(x_{t+1})] - \PE[f(x_t)] \leq & -\eta(1-\frac{\rho}{2}) \PE[\Vert \nabla f (\param_t)\Vert^2] + \frac{\eta L^2}{2\rho} \frac{2\delta(1+\delta)}{(1-\delta)^2}\eta^2\sigma^2 \\
		& + \frac{L}{2} (1+\beta) \eta d\kappa + \frac{L}{2}\eta^2(1+\frac{1}{\beta})\sigma^2.
	\end{align*}
	Rearranging the terms and averaging over $t$ gives for $\rho < 2$ (we can choose for instance $\rho = \beta = 1$):
	\begin{align*}
		\frac{1}{T+1} \sum_{t=0}^T \PE[\Vert \nabla f(\param_t) \Vert^2] \leq \frac{2(f(w_0)-f_*)}{\eta (T+1)} + 2L\sigma^2 \eta + 2L^2\sigma^2 \frac{\delta(1+\delta)}{(1-\delta)^2} \eta^2 + 2Ld\kappa.
	\end{align*}
\end{proof}

\section{Conclusion}
The bound in \Cref{theorem:cvgquant} contains 4 terms. The first term is standard for a general non-convex target and expresses how initialization affects convergence. The second and third terms depend on the fluctuation of the minibatch gradients. Another important aspect of the rate determined by \Cref{theorem:cvgquant} is its dependence on the quantization error. Note that there is an "error bound" of $2Ld\kappa$ that remains independent of the number of update iterations. The error bound is the cost of using discrete weights as part of the optimization algorithm. Previous work with quantized models also includes error bounds \citep{li2017training,li2019dimension}.

\section{Acknowledgments}
LL would like to thank Van Minh Nguyen for the useful discussions that lead to the idea of this project. We also thank Youssef Chaabouni for discussions, fixes and suggestions on manuscript writing.

\clearpage
\newpage
\bibliography{biblio}

\end{document}